\documentclass[final]{siamart171218}

\usepackage{braket,amsfonts}
\usepackage{array}
\usepackage{pgfplots}
\pgfplotsset{compat=newest}
\usepackage[caption=false]{subfig}
\usepackage{mathtools}
\usepackage{eqparbox}
\usepackage{enumitem}
\usepackage{etoolbox}

\newsiamthm{claim}{Claim}
\newsiamremark{conjecture}{Conjecture}
\newsiamremark{remark}{Remark}
\newsiamremark{expl}{Example}
\newsiamremark{hypothesis}{Hypothesis}
\newsiamremark{problem}{Problem}
\crefname{hypothesis}{Hypothesis}{Hypotheses}
\AtEndEnvironment{problem}{\null\hfill\proofbox}

\newtheorem{assumption}{Assumption}

\usepackage{algorithmic}

\usepackage{graphicx,epstopdf}
\usepackage{float}

\Crefname{ALC@unique}{Line}{Lines}

\usepackage{amsopn}

\usepackage{verbatim}

\usepackage{MnSymbol}

\usepackage{ulem}

\usepackage{dsfont} 

\newcommand{\lvertiii}{{\vert\kern-0.25ex\vert\kern-0.25ex\vert}}    
\newcommand{\rvertiii}{{\vert\kern-0.25ex\vert\kern-0.25ex\vert}}

\reversemarginpar
\usepackage[colorinlistoftodos,prependcaption,textsize=tiny]{todonotes}

\newcommand{\mcl}{\mathcal}

\newcommand{\mbb}{\mathbb}

\newcommand{\dd}{\text{d}}

\newcommand{\Lip}{{\rm Lip}_1}

\newcommand{\veps}{\varepsilon}
\newcommand{\R}{\mbb{R}}

\newcommand{\mX}{\mcl{X}}
\newcommand{\mY}{\mcl{Y}}
\newcommand{\PP}{\mbb{P}}
\newcommand{\EE}{\mbb{E}}

\DeclareMathOperator*{\Id}{id}

\usepackage{mhequ}

\newcommand{\be}{\begin{equs}}
\newcommand{\ee}{\end{equs}}
\newcommand{\bpm}{\begin{pmatrix}}
\newcommand{\epm}{\end{pmatrix}}
\DeclareMathOperator{\E}{\mathbb E}

\newcommand{\cc}{c}
\newcommand{\K}{\mathcal{D}}
\newcommand{\D}{\mathcal{D}}
\newcommand{\W}{\mathcal{W}}
\newcommand{\N}{\mathsf{N}}

\newcommand{\Sep}{\:\: \big| \:}

\newcommand{\Law}{{\rm Law}}
\newcommand{\iidsim}{\stackrel{i.i.d.}{\sim}}

\title{Bayesian Posterior Perturbation Analysis with Integral 
Probability Metrics}

\author{Alfredo Garbuno-I{\~n}igo
\thanks{Instituto Tecnol{\'o}gico Aut{\'o}nomo de M{\'e}xico, M{\'e}xico city, M{\'e}xico, \texttt{(alfredo.garbuno@itam.mx)}} \and
 Tapio Helin 
 \thanks{Lappeenranta-Lahti University of Technology, Lappeenranta, Finland, \texttt{(tapio.helin@lut.fi)}) } \and
 Franca Hoffmann 
 \thanks{California Institute of Technology, Pasadena, CA, USA, \texttt{franca.hoffmann@caltech.edu)}} \and 
 Bamdad Hosseini 
 \thanks{University of Washington, Seattle, WA, USA, \texttt{(bamdadh@uw.edu)}}
 }

\begin{document}
\date{today}
\maketitle


\begin{abstract}
In recent years, Bayesian inference in large-scale inverse problems found in science, engineering and machine learning has gained significant attention. This paper examines the robustness of the Bayesian approach by analyzing the stability of posterior measures in relation to perturbations in the likelihood potential and the prior measure.  We present new stability results using a family of integral probability metrics (divergences) akin to dual problems that arise in optimal transport. 
Our results stand out from previous works 
in three directions: (1) We construct new families of integral probability metrics
that are adapted to the problem at hand; (2) These new metrics allow us to 
study both likelihood and prior perturbations in a convenient way; and (3) 
our analysis accommodates 
likelihood potentials that are only locally Lipschitz, making them applicable to a wide range of nonlinear inverse problems. Our theoretical findings are further reinforced through specific and novel examples where the approximation rates of posterior measures are obtained for different types of 
 perturbations and provide a path towards the convergence analysis of recently adapted 
machine learning techniques for Bayesian inverse problems such as 
data-driven priors and neural network surrogates.
\end{abstract}




\section{Introduction}
Bayesian inverse/inference problems (BIPs) are ubiquitous in science and engineering applications
where uncertainty quantification (UQ) is crucial \cite{le2010spectral}. With the rise of machine learning and
data science in recent years, BIPs have found new applications in these domains specially
since  UQ enabled learning algorithms are becoming increasingly popular 
\cite{abdar2021review, arridge2019solving, gonzalez2022solving, kovachki2020conditional, laumont2022bayesian, ongie2020deep, patel2022solution}.
These new application domains for BIPs pose exciting and new theoretical questions
pertaining to the well-posedness and stability of BIPs which in turn require the
development of new tools.
The goal of this article is to address such questions focusing on the perturbation
properties of Bayesian posterior measures with
respect to prior measure, the likelihood and the data. Many learning problems concern
high-dimensional parameter spaces, such as images, that can be naturally viewed as
functional data making the framework of function space BIPs a natural choice. 

Let us consider BIPs defined on a separable Banach space $\mX$
with $\PP(\mX)$ denoting the space of Radon probability measures on $\mX$. 
We consider  {\it posterior} measures $\nu \in \PP(\mX)$ 
defined via their density with respect to an underlying {\it prior} $\mu \in \PP(\mX)$:
  \begin{equation}\label{bayes-rule}
  \frac{\dd \nu}{\dd \mu}(u) = \frac{1}{Z(y)} \exp( - \Phi(u; y)), \quad
  Z(y) :=  \mu \left(  \exp(- \Phi(u;y) \right),
\end{equation}
where we used the notation $\mu(f) := \int_\mX f(y) \dd \mu(u)$. 
We refer to  $\Phi: \mX \times \mY \mapsto \mbb{R}$ as the {\it likelihood potential},
$\mY$ is the space for the observed data $y$, 
and $Z(y)$ is the normalizing constant  that ensures $\nu$ is a 
probability measure (see \cite{stuart-acta-numerica} for a derivation of
Bayes' rule in the function space setting).

In this article we are concerned with the perturbation properties of the posterior
measure $\nu$ with respect to the likelihood function $\Phi$,
 the data $y$, and the prior $\mu$. More precisely, let $\mu_\veps \in \PP(\mX)$, $\Phi_\veps$,
and $y_\veps$ denote perturbations of the prior, likelihood, and the data, parameterized
by  $\veps \in [0, +\infty)$ giving rise to a posterior measure $\nu_\veps$
via \eqref{bayes-rule}. Then we want to show that as $\veps \to 0$ the posterior $\nu_\veps \to \nu$
in an appropriate sense. Indeed, we will  choose a statistical divergence
$\D: \PP(\mX) \times \PP(\mX) \to \R_{\ge 0}$ and establish inequalities of the form
$\D(\nu, \nu_\veps) \lesssim \psi(\veps)$, for some function $\psi$, thereby obtaining a
rate of convergence as well. 

Such perturbation results 
are closely linked with applications of BIPs for the inference of high-dimensional 
or functional parameters. For example likelihood perturbations arise naturally in 
function space inference where a discretization scheme should be employed before 
any simulations can be done or in the context of surrogate models 
such as polynomial chaos expansions \cite{marzouk2007stochastic, schwab-sparse, marzouk2009dimensionality, ma2009efficient}, Gaussian process regression \cite{stuart2018posterior, teckentrup2020convergence} or neural networks \cite{herrmann2020deep, li2020fourier} where the likelihood (or the underlying forward model) 
are replaced  with a computationally cheap approximation.
Indeed, likelihood perturbations 
 are a well-studied subject in the {\it well-posedness theory of  BIPs}
\cite{stuart-bayesian-lecture-notes, stuart-acta-numerica, hosseini-convex,
  hosseini-sparse, latz, sullivan, owhadi2015brittleness, sprungk}.
  However,
with the exception of \cite{sprungk}, the previous results overlooked
perturbations of the prior measure $\mu$. This is largely due to the
choice of $\D$, as convergence in total variation (TV) or
Hellinger metrics requires the posterior $\nu$ and the perturbation $\nu_\veps$
to be equivalent to obtain meaningful rates. This, in turn, requires the prior $\mu$ and the perturbation
$\mu_\veps$ to be equivalent as well, but this requirement
is too strong and excludes interesting cases that arise in 
modern applications such as data-driven construction of priors 
\cite{arridge2019solving, laumont2022bayesian, patel2022solution}, i.e., the 
case where the prior is learned from a training data set. For example, we might take $\mu_\veps$
as an empirical approximation to $\mu$ obtained from a set of independent samples. 
Another popular approach is to parameterize 
$\mu_\veps = T^\veps_\sharp \eta$
(the pushforward of $\eta$ through $T^\veps$)
where $T^\veps$ is a map (such as a neural network)  that is learned offline
 and $\eta$ is a generic reference measure on an abstract latent space. 
 To see why this approach can easily lead to singular priors consider 
the simple example where $\mu$ is a Gaussian  and $\mu_\veps = \mu( \cdot - \veps v)$ (the transport
map is simply $T^\veps: x \mapsto x + \veps v $)
for an element $v \in \mX$ that is  outside of the Cameron-Martin space of $\mu$ \cite[Sec.~2.4]{bogachev-gaussian} in which case $\mu$ and $\mu_\veps$ become mutually singular.

One of the novelties of \cite{sprungk} was to take $\D$ to be the Wasserstein metric $\W_p$.
These metrics are flexible enough to allow for
controling posterior perturbations in terms of prior perturbations but
one still needs to assume that the likelihood $\Phi$ is globally Lipschitz,
which excludes many nonlinear inverse problems. Here we will get around
this issue by taking $\D$ to be an integral probability metric (IPM)
\footnote{We call these integral probability 'metrics' since this is 
the established terminology in the literature although our IPMs are often 
only divergences or semi-metrics and do not necessarily satisfy all of the axioms of a metric.}:
\begin{equation}\label{IPD}
  \D(\eta_1, \eta_2) = \sup_{\psi \in \Gamma} \: \eta_1( \psi)
  - \eta_2( \psi), \qquad \forall  \eta_1,\eta_2 \in \PP(\mX),
\end{equation}
where $\Gamma$ is a subset of real valued functions on $\mX$ \footnote{We will primarily 
focus on the case where $\Gamma$ is a class of functions that are Lipschitz with respect to a 
cost function $c$ as these are sufficient for our applications, but our results can be generalized to broader choices.}.
The idea is that if $\Gamma$ is sufficiently large then one can maximize
the right hand side by choosing a $\psi$ that can distinguish any distinct pair
of measures in $\PP(\mX)$ so that $\D(\eta_1, \eta_2) = 0$ if and only if
$\eta_1 = \eta_2$; this is referred to as the {\it divergence property} of $\D$ in
the parlance of \cite{birrell2022f}. We note that while the divergence property
is theoretically attractive, it is not necessary in many UQ applications. For example,
let $\psi^\star$ be a quantity of interest whose expectation with respect to $\nu$
is desired. Then we simply need to make sure that the set $\Gamma$ contains $\psi^\star$
so that we can trivially bound 
\begin{equation*}
  |\nu(\psi^\star) - \nu_\veps(\psi^\star)| \le \D(\nu, \nu_\veps),
\end{equation*}
which allows us to control the error of the quantity of interest.

IPMs have been used extensively
in the statistical theory \cite{muller1997integral} 
and more  recently in  machine learning \cite{sriperumbudur2009integral, birrell2022f}; 
popular examples include the maximum 
mean discrepancy (MMD) \cite{muandet2017kernel, arbel2019maximum} and 
the $f$-divergences of \cite{nowozin2016f}.  
IPMs further appear  naturally in connection with optimal transport (OT) and 
in the context of the dual form of Wasserstein metrics \cite{villani-OT}; a property that 
we will exploit in this work as well.
At the same time, closely related
OT semi-metrics have been used in the
convergence analysis of stochastic PDEs and function space
Markov chain Monte Carlo algorithms
\cite{hairer2011asymptotic, hairer2014spectral, johndrow-approximation, hosseini-spectralgapARSD}; these 
works were a major inspiration for our analysis.
Here we take advantage of the flexibility of IPMs 
and in particular adapt the choice of the space $\Gamma$ 
to the likelihood  potential $\Phi$ and obtain perturbation 
bounds in more general settings than previous works. 
This idea of adapting the IPM to the likelihood at hand is the key to
simplifying a lot of our arguments and leads to 
flexible bounds.

\subsection{Main Contributions}\label{sec:main-contributions}

Our main contributions are summarized below: 


\begin{itemize}
\item 
We measure the approximation rate of the posterior using 
abstract IPMs, which is a broad category of divergences for probability measures that encompasses, among others, the Wasserstein metric. Our approach adapts the choice of the IPM to the likelihood of the 
problem at hand and does not require any global 
Lipschitz-stability assumptions on the log-likelihood.
Theorems  \ref{thm:posterior-likelihood-perturbations-general} and \ref{thm:posterior-prior-perturbations-general} analyze posterior perturbations with 
regards to likelihood and prior perturbations, respectively while
 Corollaries \ref{data-pert-cor-1} and \ref{data-pert-cor-2}  examine perturbations of the data.

\item Multiple concrete applications in Section~\ref{sec:applications} demonstrate the power of our 
perturbation analysis. Here we obtain quantitative 
convergence rates for posterior measures that arise 
from approximations of priors including 
empirical approximations of the prior from 
a set of samples in Section~\ref{sec:empirical-prior};   
spectral perturbations of product priors such as 
Gaussian measures with Mat{\'e}rn kernels with varying 
hyper-parameters in Section~\ref{sec:matern-prior}; 
data-driven priors identified as  pushforwards of 
parameterized transport maps in Section~\ref{sec:pushforward-prior};
and deep neural net (DNN) 
surrogate models of the likelihood in Section~\ref{sec:surrogate-models}.

\end{itemize}

\subsection{Outline}\label{sec:outline} 
The rest of the paper is organized as follows: In Section~\ref{sec:prelims} we discuss mathematical preliminaries and IPMs. 
Our stability results are derived in Section~\ref{sec:pert-theory-post} while
applications are outlined in  Section~\ref{sec:applications}.

\section{Distance-like Functions and IPMs}\label{sec:prelims}
We collect here some basic properties and definitions pertaining to IPMs that will be used to obtain our perturbation results in Section~\ref{sec:pert-theory-post}.  We say that a function $\cc: \mX \times \mX \mapsto \mbb{R}$ is
{\it distance-like}  in the parlance of  \cite{hairer2011asymptotic} if 
it is positive, symmetric, lower-semicontinuous and such that $\cc(u,v) = 0$ if and only 
if $u =v$. Given such a distance-like  $\cc$ we define the
following subspace of 
$\PP(\mX)$:
\begin{equation*}
  P_1(\mX; c) : = \left\{ \eta \in \PP(\mX) \Sep \eta(  \cc(0, u) )  < +\infty  \right\}, 
\end{equation*}
as well as  the space of Lipschitz-1 functions on $\mX$ with respect to $\cc$,
\begin{equation*}
  \Lip(\mX; \cc) := \left\{  \psi: \mX \mapsto \mbb{R} \Sep \sup_{u\neq v}
    \frac{|\psi(u) - \psi(v)| }{\cc(u,v)} \le 1 \right\}.
\end{equation*}
We also write $\Lip^0(\mX; \cc)$ to denote the subspace of $\Lip(\mX; \cc)$ consisting of
functions that vanish at the origin. 
Finally, we define the following IPM resembling the Kantorovich-Rubinstein
functional \cite{bogachev2012monge}:
\begin{equation}\label{d-KR}
  \begin{aligned}
  \D(\eta_1, \eta_2; \cc) & :=   
  \sup_{\psi \in \Lip(\mX ; \cc)}  \eta_1( \psi) -  \eta_2( \psi),\\
  & \equiv \sup_{\psi \in \Lip^0(\mX ; \cc)}  \eta_1( \psi) -  \eta_2( \psi), \qquad \eta_1, \eta_2 \in \PP(\mX),
\end{aligned}
\end{equation}
where the equivalence between the two definitions is simply due to the fact that 
the value of the functional on the right hand side does not change if we 
shift $\psi$ by a constant.
It is known (see for example \cite[Particular Case 5.4]{villani-OT}) that if
$c$ satisfies a triangle inequality then 
the functional $\D$ satisfies a
duality theorem with respect to an OT cost function, i.e.,
\begin{equation}\label{K-W-duality}
  \D(\eta_1, \eta_2; \cc) = \W( \eta_1, \eta_2; \cc) := \inf_{\pi \in \Pi( \eta_1, \eta_2)}
  \int_{\mX \times \mX} c(u,v) \dd \pi(u, v),
\end{equation}
where $\Pi(\eta_1, \eta_2)$ is the space of all couplings between $\eta_1, \eta_2$, i.e., the
subspace of $\PP(\mX \times \mX)$ consisting of measures with marginals $\eta_1, \eta_2$. 
This further implies that if $\cc$ is a metric on $\mX$ then $\D$ defines a metric
on $P_1(\mX; \cc)$ \cite{bogachev2012monge,villani-OT}.
In the specific case where
$\cc(u,v) = \| u -v \|_\mX^p$  for $p \ge 1$ the functional $\K^{1/p}$ coincides with the well-known Wasserstein-$p$ metric for probability measures.

In Section~\ref{sec:pert-theory-post} we will 
need to work with cost functions $c$ that are not metrics but semi-metrics in the sense that 
they satisfy a weak triangle inequality (triangle inequality with a constant that is larger than one). 
In this case the functional $\D$ may no longer satisfy the duality \eqref{K-W-duality}, but 
 we still have access to the inequality 
\begin{equation}\label{D-le-W}
  \D(\eta_1, \eta_2; c) \le \W(\eta_1, \eta_2; c),
\end{equation}
thanks to the Lipschitz assumption on $\psi$. This identity will be very useful in 
Section~\ref{sec:applications}
where we use coupling arguments or known bounds on Wasserstein distances to bound $\D$. 
Furthermore, since we assumed $c$ is distance-like 
we can verify that $\W(\cdot, \cdot; c)$ satisfies 
the {\it divergence property} in the parlance of \cite{birrell2022f}, i.e., 
$\W(\eta_1, \eta_2; c) = 0$ if and only if $\eta_1 = \eta_2$. 
Then it follows from \eqref{D-le-W} that $\D(\eta_1, \eta_2; c) =0$ 
whenever $\eta_1 = \eta_2$. However, to show the converse result, and 
hence the divergence property of $\D$ we need an additional assumption on 
$\cc$. 

\begin{proposition}
  Suppose $c$ is distance-like and there exists a metric $\rho$  such that  $\rho(u, v) \le c(u,v)$, then $\D(\cdot, \cdot;c)$ 
  satisfies the divergence property.
\end{proposition}

\begin{proof}
Following our discussion above we only need to show that $\eta_1 \neq \eta_2$
implies $\D(\eta_1, \eta_2; c) >0$. We have
$\Lip(\mX; \rho) \subset \Lip(\mX; c)$ which together with the 
duality result for $\W(\cdot, \cdot; \rho)$ yields 
$\W(\cdot, \cdot; \rho) \le \D(\cdot, \cdot; c)$. The result follows 
since $\W(\cdot, \cdot; \rho)$ is a metric and therefore satisfies the divergence property.
\end{proof}

\section{Perturbation Theory for Posterior Measures}\label{sec:pert-theory-post}

In this section, we outline our main  theoretical contributions and show the 
Lipschitz stability of posteriors with respect to likelihood and prior perturbations in appropriate 
IPMs. More precisely, given the ground truth posterior $\nu$ as in \eqref{bayes-rule}
we can consider the measures
\begin{equation*}
    \frac{\dd \nu_\ast}{\dd \mu_\ast}(u) = \frac{1}{Z_\ast(y)} \exp \left( - \Phi(u; y) \right), 
    \quad \text{and} \quad 
    \frac{\dd \nu'_\ast}{\dd \mu_\ast}(u) = \frac{1}{Z_\ast'(y)} \exp \left( - \Phi'(u; y) \right), 
\end{equation*}
with the normalizing constants $Z_\ast(y)$ and $Z'_\ast(y)$ defined analogously to \eqref{bayes-rule}. Then if $\K$ is an IPM 
that satisfies a weak triangle inequality we can readily write 
\begin{equation*}
    \K( \nu, \nu'_\ast) \lesssim \K(\nu, \nu_\ast) + \K(\nu_\ast, \nu'_\ast),
\end{equation*} 
where the first term in the right hand side only involves prior perturbations while 
the second term concerns likelihood perturbations. Therefore, it is natural for us 
to consider each of these perturbations separately in this section.
We begin by outlining our main assumptions on the likelihood in Section~\ref{sec:likelihood-assumptions} followed by 
likelihood perturbations, including the case of perturbed data, in
Section~\ref{sec:pert-likel} and prior perturbations  in Section~\ref{sec:pert-prior}.

\subsection{Assumptions on The Likelihood}\label{sec:likelihood-assumptions}

Below we gather a set of assumptions on the likelihood potential $\Phi$ that are used throughout the rest of the
article. 
These are standard assumptions in
well-posedness theory of BIPs \cite{hosseini-convex,hosseini-sparse,latz,sprungk,stuart-acta-numerica, sullivan} that 
ensure the resulting posterior is well-defined. 

\begin{assumption}\label{assumptions-on-phi}
  The likelihood potential $\Phi: \mX \times \mY \mapsto \mbb{R}$ satisfies one or more of the following 
  assumptions : 
\begin{enumerate}[label={(\roman*)}]

\item ({\it Prior measurability}) The function $\Phi(\cdot; y) : \mX \mapsto \mbb R$
  is measurable with respect to the prior $\mu$ for any fixed $y \in \mY$.

\item ({\it Locally bounded  above and below}) There exists a
  lower-semicontinuous function
  $f: \mX \to \mbb{R}_{>0}$ and
  locally bounded functions $g: \mY \to \mbb{R}_{>0}$ and
  $h: \mX \times \mY \to \mbb{R}_{>0}$ so that
  $\forall (u,y) \in \mX \times \mY$
  it holds that
  \begin{equation*}
    -  \log f (u) - \log g(y)  \le \Phi(u;y) \le - \log h(u,y).
\end{equation*}

\end{enumerate}
\end{assumption}

With the above assumptions at hand we can immediately guarantee the existence and
uniqueness of the posterior measure $\nu$ for fixed data $y \in \mY$.

\begin{proposition}[{\cite[Thm.~4.3]{hosseini-sparse}}]
  \label{well-defined-posterior}
  Suppose $\Phi$ satisfies Assumptions~\ref{assumptions-on-phi} with
  $f \in L^1(\mX, \mu)$.
  Then for any fixed $y \in \mY$
 we have
$0<Z(y)< + \infty$ implying that   
$\nu$ is a well-defined probability measure.
\end{proposition}

We highlight that while our assumptions are sufficient for ensuring the posterior is 
well-defined they are certainly not necessary. Indeed \cite{latz} obtains the same 
result with much weaker assumptions but we impose stronger assumptions, in particular 
condition (ii), in order to construct our IPMs and obtain our stability results.
At the same time condition (ii) is sufficiently general to encompass many interesting applications 
such as those outlined in Section~\ref{sec:applications}.

\subsection{Perturbations of the Likelihood}\label{sec:pert-likel}

In this section we study the perturbations of the posterior measure $\nu$ with respect to
perturbations of the likelihood $\Phi$. Such perturbations arise 
naturally in the discretization of function space BIPs and
 surrogate modelling of expensive likelihoods \cite{marzouk2007stochastic, schwab-sparse, marzouk2009dimensionality, ma2009efficient, stuart2018posterior, teckentrup2020convergence, herrmann2020deep, li2020fourier}.

Consider the modified
likelihood potential $\Phi'( u; y)$ and in turn define the perturbed posterior $\nu' \in \PP(\mX)$: 
\begin{equation}
  \label{Bayes-rule-perturbed-likelihood}
  \frac{\dd \nu'}{\dd \mu} = \frac{1}{Z'(y)} \exp \left( - \Phi'(u; y) \right),
  \qquad Z'(y) := \int_\mX \exp( - \Phi'(u;y)) \dd \mu(u).
\end{equation}
We develop our perturbation theory in this general setting and 
consider two specific settings later where
either the data $y$ is perturbed or the likelihood potential $\Phi$ is itself modified
due to approximations.

\begin{theorem}\label{thm:posterior-likelihood-perturbations-general}
  Suppose $\Phi$ and $\Phi'$ satisfy Assumption~\ref{assumptions-on-phi} 
  with the same functions  $f,g ,h$,
  and let $p, q \in [1, \infty]$
  satisfy $1/p + 1/q = 1$. 
  Then it holds that
  \begin{equation*}
    \K(\nu, \nu'; \cc) \le
    \frac{ 2 g^2(y) \| f c(\cdot, 0) \|_{L^p(\mu)} \| f \|_{L^p(\mu)}}{\| h(\cdot, y) \|_{L^1(\mu)}^2}
  \| \Phi(\cdot; y) - \Phi'(\cdot; y) \|_{L^q(\mu)}\,,\quad \text{ for all } y\in\mY\,.
    \end{equation*}
  \end{theorem}

  \begin{remark}
    Note that the requirement that $\Phi'$ satisfies Assumption~\ref{assumptions-on-phi} with
    the same functions $f, g, h$ as $\Phi$ is innocuous. In fact, if $\Phi'$ satisfies
    that assumption with different functions $f', g', h'$ then
    Theorem~\ref{thm:posterior-likelihood-perturbations-general} holds by replacing
    $f$ with
    $f \vee f'$, $g$ with $g \vee g'$, and $h(\cdot, y)$ with $h(\cdot, y) \wedge h'(\cdot, y)$ 
    \footnote{Given $a, b \in \R$ we write $a \vee b$ to denote 
    their maximum and $a \wedge b$ to denote their minimum.}.
    Moreover, since both $\nu$ and $\nu'$ are absolutely continuous with respect
    to $\mu$ one can relax Assumption~\ref{assumptions-on-phi}(ii) to hold
    for $\mu$-a.e. $u$ rather than all of $\mX$.
\end{remark}

\begin{proof}
  Our proof follows similar steps to the proof of \cite[Thm.~14]{sprungk}.  
  By the definition of $\K$ and the measures $\nu, \nu'$ we have 
  \begin{equation*}
    \begin{aligned}
    \K( \nu, \nu'; \cc) & = \sup_{\psi \in \Lip^0(\mX; \cc)}
    \frac{1}{Z(y)} \int_\mX \psi(u) \exp( - \Phi(u;y)) \dd \mu(u) \\
    & \qquad \qquad  \qquad -  \frac{1}{Z'(y)} \int_\mX \psi(u) \exp( - \Phi'(u;y)) \dd \mu(u) \\
    & \le \frac{1}{Z'(y)}  \sup_{\psi \in \Lip^0(\mX; \cc)} \int_\mX \psi(u)
    \left[ \exp(-\Phi(u; y)) - \exp(- \Phi'(u;y) ) \right] \dd \mu(u)  \\
    & \quad  + \frac{| Z(y) - Z'(y)|}{Z(y) Z'(y)} \sup_{\psi \in \Lip^0(\mX; \cc)} \int_\mX \psi(u) \exp( - \Phi(u;y)) \dd \mu(u)
    =: T_1 + T_2.   
   \end{aligned}
 \end{equation*}
 First, we bound $T_1$ using the mean value theorem for the exponential function
 along with the fact that Assumption~\ref{assumptions-on-phi}(ii) is satisfied by
 $\Phi$ and $\Phi'$. 
 \begin{equation*}
   \begin{aligned}
     Z'(y) T_1
     & \le \sup_{\psi \in \Lip^0(\mX; \cc)}
     \int_\mX | \psi (u) | \exp \left( - \big[ \Phi(u;y) \wedge \Phi'(u;y) \big] \right)
     | \Phi(u;y) - \Phi'(u;y) | \dd \mu(u) \\
     & \le g(y)
        \sup_{\psi \in \Lip^0(\mX; \cc)} \int_\mX |\psi (u)| f( u)  | \Phi(u;y) - \Phi'(u;y) | \dd \mu(u).
   \end{aligned}
 \end{equation*}
 Now using the fact that functions in $\Lip^0(\mX; \cc)$ are bounded by $\cc(\cdot , 0)$ and
 H{\"o}lder's inequality we can further bound $T_1$ as follows:
 \begin{equation*}
   \begin{aligned}
     Z'(y)  T_1  &  \le  g(y)
     \int_\mX \cc(u, 0) f(u) | \Phi(u;y) - \Phi'(u;y) | \dd \mu(u) \\
     &\le g(y) \| f \cc(\cdot, 0) \|_{L^p(\mu)} 
     \left\| \Phi(\cdot ;y) - \Phi'(\cdot; y) \right\|_{L^q(\mu)}.  
   \end{aligned}
 \end{equation*}
 By applying similar arguments to $T_2$ we can write
 \begin{equation*}
   \begin{aligned}
     Z(y) Z'(y) T_2
     & \le | Z(y) - Z'(y)|  g(y)  \int_\mX c(u, 0) f(u) \dd \mu(u) \\
     & \le | Z(y) - Z'(y)| g(y) \| f c(\cdot, 0) \|_{L^p(\mu)}.
 \end{aligned}
 \end{equation*}
 Next we bound $| Z(y)- Z'(y)|$.
 \begin{equation*}
   \begin{aligned}
   | Z(y) - Z'(y)| & = \left| \int_\mX \exp( - \Phi(u;y) ) 
     - \exp( - \Phi'(u;y) )  \dd \mu(u)
   \right| \\
   & \le \left| \int_\mX \exp \left( - \big[\Phi(u;y) \wedge \Phi'(u;y)  \big] \right)
       | \Phi(u;y) - \Phi'(u;y) | \dd \mu  \right| \\
   & \le g(y)   \left| \int_\mX f(u) |  \Phi(u;y)  - \Phi'(u;y) |  \dd \mu(u)
   \right| \\
   & \le g(y)  \| f \|_{L^p(\mu)} \left\|  \Phi(\cdot ;y)  - \Phi'(\cdot;y) \right\|_{L^q(\mu)}.
 \end{aligned}
\end{equation*}
Putting this bound together with our previous bound on $T_2$ yields,
\begin{equation*}
  T_2 \le \frac{g^2(y)}{Z(y) Z'(y)}  \| f\|_{L^p(\mu)} \| f c(\cdot, 0) \|_{L^p(\mu)}
  \| \Phi(\cdot; y) - \Phi'(\cdot; y) \|_{L^q(\mu)}.
\end{equation*}
By combining the bounds for $T_1, T_2$ we then obtain
\begin{equation}\label{K-likelihood-pert-prelim-bound}
  \K( \nu, \nu'; c)
  \le \frac{g(y) \| f c(\cdot, 0) \|_{L^p(\mu)} \Big[ Z(y) +  g(y) \| f \|_{L^p(\mu)}  \Big]}{Z(y) Z'(y)}
  \| \Phi(\cdot; y) - \Phi'(\cdot; y) \|_{L^q(\mu)}.
\end{equation}
However, since $\Phi$ satisfies  Assumption~\ref{assumptions-on-phi}(ii) we have that
\begin{equation}\label{upper-bound-on-Z-y}
  Z(y)  \le g(y)  \int_\mX  f(u)  \dd \mu(u) \le g(y) \| f\|_{L^1(\mu)} \le  g(y) \| f\|_{L^p(\mu)},
\end{equation}
and since Assumption~\ref{assumptions-on-phi}(ii) is also satisfied by $\Phi'$ we
obtain the lower bound 
\begin{equation*}
  Z(y) Z'(y) \ge
  \left( \int_\mX h(u,y) \dd \mu(u) \right)^2 = \| h(\cdot, y) \|_{L^1(\mu)}^2.
\end{equation*}
Substituting these bounds back into \eqref{K-likelihood-pert-prelim-bound} yields the desired
result.
{}
\end{proof}

The above proof immediately reveals the following corollary.

\begin{corollary}\label{cor:posterior-likelihood-perturbations-general}
        The result of Theorem~\ref{thm:posterior-likelihood-perturbations-general} is also 
        true with the bound \eqref{K-likelihood-pert-prelim-bound} which is 
        explicit in terms of the constants $Z(y)$ and $Z'(y)$.
\end{corollary}

\begin{remark}\label{rem:brittleness}
Equation \eqref{K-likelihood-pert-prelim-bound} also reveals a delicate property of 
our bounds, namely that as $Z(y)$ or $Z'(y)$ vanish the Lipschitz constants in our 
bounds blow up. This phenomenon 
was also observed in \cite{sprungk} and is a manifestation of brittleness of BIPs
\cite{owhadi2015brittleness, owhadi2016brittleness, owhadi2015brittleness-b}, i.e., if the model is mis-specified in the sense that the 
likelihood is concentrated in an area of low prior probability, then 
the posterior measure is very sensitive to perturbations of the likelihood, the data, 
and the prior. 
\end{remark}

We also note that Theorem 
\ref{thm:posterior-likelihood-perturbations-general} extends previous works also by utilizing $L^q(\mu)$-norm with $q\in [1,\infty]$ in the upper bound instead of mixed $L^1(\mu)$ and $L^2(\mu)$ norms in \cite{sprungk}. In particular, the $L^\infty$-case opens the door to the study of perturbations obtained by neural network surrogate models such as those studied in \cite{herrmann2020deep} and discussed in Section~\ref{sec:surrogate-models}.

\subsubsection{Data Perturbations}\label{sec:data-perturbation}
In this section we discuss some corollaries of Theorem~\ref{thm:posterior-likelihood-perturbations-general}
that are pertinent to applications where the perturbation in the likelihood
potential $\Phi$ is due to perturbations of the data $y$. Proving the continuous dependence of the
posterior on the data $y$ is a central question in  well-posedness theory of BIPs
\cite{stuart-acta-numerica,hosseini-convex,hosseini-sparse,sullivan} establishing
the continuity of the data to posterior  map.

To this end, we consider  $y, y' \in \mY$, thinking of $y'$ as a perturbation of
$y$. Then applying Theorem~\ref{thm:posterior-likelihood-perturbations-general}
with $\Phi'(u;y) \equiv \Phi( u; y + (y' -y))$ immediately yields
the following:

\begin{corollary}\label{data-pert-cor-1}
  Suppose $\Phi$ satisfies Assumption~\ref{assumptions-on-phi}  and
  let $p,q \in [1, \infty]$ satisfy $1/p + 1/q = 1$.
  For fixed $y \in \mY$ and $r >0$ consider the posterior measures
  \begin{equation*}
    \frac{\dd \nu'}{\dd \mu} (u) = \frac{1}{Z(y')} \exp( - \Phi(u; y')), \qquad \forall y' \in B^\mY_r(y),
  \end{equation*}
  where $B^\mY_r(y) \subset \mY$ denotes the ball of radius $r$ centered at $y$.
  Then it holds that
  \begin{equation*}
    \K(\nu, \nu'; c) \le \frac{2 \left( \sup_{z \in B^\mY_r(y)}  g^2(z) \right)
      \| f c(\cdot, 0) \|_{L^p(\mu)} \| f \|_{L^p(\mu)} }{ \| \inf_{z \in B^\mY_r(y)} h(\cdot, z) \|_{L^1(\mu)}^2}
    \| \Phi(\cdot; y) - \Phi(\cdot; y') \|_{L^q(\mu)}.
  \end{equation*}
\end{corollary}

In many applications the likelihood potential $\Phi$ satisfies stronger regularity constraints
that allow  us to control  $\K(\nu, \nu'; c)$ in terms of the size of the
perturbation in the data. 

\begin{corollary}
	\label{data-pert-cor-2}
 Suppose the conditions of Corollary~\ref{data-pert-cor-1} are satisfied
  and in addition,
  $\forall y' \in B^\mY_r(y)$ there 
  exists a function $b: \mX \to \mbb R_{\ge 0}$ so that for $\mu$-a.e. $u$,
  \begin{equation*}
    | \Phi(u; y) - \Phi(u; y') | \le b(u) \| y - y'\|_\mY, \qquad \forall y' \in B^\mY_r(y). 
  \end{equation*}
  Then it holds that
  \begin{equation*}
      \K(\nu, \nu'; c) \le \frac{2 \left( \sup_{z \in B^\mY_r(y)}  g^2(z) \right)
      \| f c(\cdot, 0) \|_{L^p(\mu)} \| f \|_{L^p(\mu)} \| b \|_{L^q(\mu)} }{ \| \inf_{z \in B^\mY_r(y)} h(\cdot, z) \|_{L^1(\mu)}^2} \| y - y' \|_{\mY}.
  \end{equation*}
\end{corollary}
In particular, the above result implies that if
$\| \inf_{z \in B_r(y)} h(\cdot, z) \|_{L^1(\mu)} >0$, $f, f c(\cdot, 0) \in L^p(\mu)$, 
and $b \in L^q(\mu)$,
then the  mapping $y \mapsto \nu$  is locally Lipschitz continuous with respect to $\K$. 

\subsection{Perturbations of the Prior}\label{sec:pert-prior}
Here we turn our attention to perturbation properties of the posterior $\nu$ with
respect to the prior measure $\mu$. More precisely, we consider a prior measure
$\mu_\ast \in \PP(\mX)$ and define the posterior measure $\nu_\ast \in \PP(\mX)$:
\begin{equation}
  \label{Bayes-rule-perturbed-prior}
  \frac{\dd \nu_\ast}{\dd \mu_\ast} = \frac{1}{Z_\ast(y)} \exp \left( - \Phi(u; y) \right).
\end{equation}

Once again we develop our perturbation theory in this general setting and later make
our main theorem concrete in the context of various applications. The main difference in the case of prior perturbations
as compared to likelihood perturbations is that here we will need to adjust the 
function $c$ with respect to which $\D$ is defined
 in order to bound the distance between the posterior measures.
 More precisely, the prior and posterior perturbations should be measured with 
 respect to different IPM defined via related distance-like functions.

\begin{theorem}\label{thm:posterior-prior-perturbations-general}
  Suppose $\Phi$ satisfies Assumption~\ref{assumptions-on-phi} and in addition
     there exists a  
    $L: \mX \times \mX \times \mY \to \mbb{R}_{\ge 0}$ 
    so that
  \begin{equation*}
    | \Phi(u; y) - \Phi(v; y) | \le L(u,v; y) \cc( u, v) 
  \end{equation*}
  for all $u,v \in \mX \times \mX$ and $y \in \mY$ and $L(\cdot, \cdot; y)$ is 
  lower-semicontinuous for any fixed value of $y$.
  Then it holds that
  \begin{equation*}
    \K(\nu, \nu_\ast; \cc) \le  \frac{g^2(y) \left[ \| f\|_{L^1(\mu)} +   \| f c(\cdot, 0) \|_{L^1(\mu)} \right] }
    {\| h(\cdot, y) \|_{L^1(\mu)} \| h(\cdot, y) \|_{L^1(\mu_\ast)}}  \K(\mu, \mu_\ast; \cc_y),
  \end{equation*}
  where the new distance-like function $\cc_y : \mX \times \mX \to \mbb{R}_{\ge 0}$ is defined as
  \begin{equation}\label{def:c-ast}
    \cc_y(u,v) := \big[1 \vee \cc(u,0) \vee \cc(v,0)\big] \cdot
    \big[f(u) \vee f(v)\big] \cdot
    \big[ 1 \vee L(u,v; y) \big] \cdot \cc(u,v).
  \end{equation}
\end{theorem}

\begin{remark}
  We highlight two important facts about the new cost functions $c_y$: (i) The subscript 
  $y$ indicates the dependence of the cost on the data $y$ through the local Lipschitz 
  constant $L(\cdot, \cdot; y)$ of $\Phi$. Of course in the case of prior perturbations 
  this dependence on $y$ is innocuous since the data is assumed to be fixed. Alternatively, 
  one can replace $L(\cdot, \cdot;y)$ with $\sup_{y\in B^\mY_r(0)} L(\cdot, \cdot; y)$
  to obtain an analogous result with a cost that is uniform for all $y \in B^\mY_r(0)$.
  Indeed, the result also holds with any cost that upper bounds $c_y(u,v)$.
  (ii) The appearance of the function $f$ as well as $L$ in the definition of $c_y$ indicates that 
  the new cost $c_y$ is adapted to the likelihood $\Phi$. Observe that $f$ controls  the 
  rate of growth of $\exp( - \Phi(u;y))$ and so it is clear from expression 
  \eqref{def:c-ast} that the growth rate of  $c_y$ is tied to that of 
 the likelihood as well as  its Lipschitz constant and the original cost $c$. Intuitively, the slower the 
 likelihood grows and the more regular it is, the closer $c_y$ is to $c$.
\end{remark}

\begin{proof}
  By the definition of $\K$ we can write
  \begin{equation*}
    \begin{aligned}
    \K( \nu, \nu_\ast; \cc)
    & = \sup_{\psi \in \Lip^0(\mX; \cc)} \frac{1}{Z} \int_\mX \psi(u) \exp (-\Phi(u;y)) \dd \mu(u) \\
    & \qquad \qquad \qquad \qquad - \frac{1}{Z_\ast}  \int_\mX \psi(u) \exp (-\Phi(u;y)) \dd \mu_\ast(u) \\
    & \le \Bigg[ \frac{1}{Z_\ast} \sup_{\psi \in \Lip^0(\mX; \cc)}
       \int_\mX \psi(u) \exp (-\Phi(u;y)) \dd \mu(u) \\
    &  \qquad \qquad \qquad \qquad  - \int_\mX \psi(u) \exp (-\Phi(u;y)) \dd \mu_\ast(u) \Bigg] \\
    &  + \Bigg[ \frac{|Z - Z_\ast|}{Z Z_\ast}  \sup_{\psi \in \Lip^0(\mX; \cc)}
    \int_\mX \psi(u) \exp ( - \Phi(u;y)) \dd \mu(u) \Bigg] \\
    & =: T_1 + T_2. 
  \end{aligned}
  \end{equation*}
  Let us first bound $T_1$. For any $\psi \in \Lip^0(\mX; \cc)$ and any $u,v \in \mX$ we have
  \begin{equation*}
    \begin{aligned}
      |\psi(u) &\exp( - \Phi(u;y)) - \psi(v) \exp( - \Phi(v;y) )| \\
      & \le |\psi(u)| | \exp( - \Phi(u;y)) - \exp( - \Phi(v;y)) |  + \exp( -\Phi(u;y)) | \psi(u;y) - \psi(v;y)| \\
      & \le \cc(u, 0)   | \exp( - \Phi(u;y)) - \exp( - \Phi(v;y)) |  + \exp( -\Phi(u;y)) \cc(u, v).
    \end{aligned}
  \end{equation*}
  Now applying the mean value theorem for the exponential functions, as well as Assumption~\ref{assumptions-on-phi}(ii) and the local Lipschitz hypothesis on $\Phi(u;y)$ we can further write
  \begin{equation*}
    \begin{aligned}
      |\psi(u) &\exp( - \Phi(u;y)) - \psi(v) \exp( - \Phi(v;y) )| \\
      & \le \cc(u, 0) \exp \left( - [\Phi(u;y) \wedge \Phi(v;y) ] \right)   |  \Phi(u;y) -  \Phi(v;y) |
      + \exp( -\Phi(u;y)) \cc(u, v) \\
      & \le  g(y) \cc(u, 0)\left[ f(u) \vee f(v) \right] L(u,v; y) \cc(u,v)
      + g(y) f(u) \cc(u,v)  \\
      & \le g(y) \bigg[ \big[1 \vee \cc(u,0) \vee \cc(v,0)\big]
      \big[f(u) \vee f(v)\big]
      \big[ 1 \vee L(u,v; y) \big] \bigg] \cc(u,v) \\
      & = g(y) \cc_y(u,v).
    \end{aligned}
  \end{equation*}
  Thus,  $ g(y)^{-1} \psi(\cdot) \exp( - \Phi(\cdot; y)) \in \Lip(\mX; \cc_y)$ which immediately gives the bound
  \begin{equation*}
    Z_\ast(y) T_1 \le g(y) \K( \mu, \mu_\ast; \cc_y).
  \end{equation*}
  Proceeding to bound $T_2$, we  use  Assumption~\ref{assumptions-on-phi}(ii) once more to write
  \begin{equation*}
    \begin{aligned}
      Z(y) Z_\ast(y) T_2
      & \le \left[ g(y)  \int_\mX c(u, 0) f(u) \dd \mu(u)  \right] | Z(y) - Z_\ast(y)| \\
      & =  g(y)\| f c(\cdot, 0) \|_{L_1(\mu)} | Z(y) - Z_\ast(y)|.
  \end{aligned}
\end{equation*}
Furthermore, by  Assumption~\ref{assumptions-on-phi}(ii), the Lipschitz hypothesis on $\Phi$ and
the definition of  $c_y(u,v)$, we can directly verify that $ g(y)^{-1} \exp( - \Phi(\cdot;y) ) \in \Lip(\mX; \cc_y)$
and so
\begin{equation*}
  | Z(y) - Z_\ast(y) | \le g(y) \K( \mu, \mu_\ast ; \cc_y),
\end{equation*}
which in turn gives the final bound
\begin{equation*}
  T_2 \le \frac{ g^2(y) \| f c( \cdot, 0 ) \|_{L^1(\mu)}}{Z(y) Z_\ast(y)} \K( \mu, \mu_\ast ; \cc_y).
\end{equation*}
  Combining the bounds on $T_1, T_2$ then yields
\begin{equation}\label{K-prior-pert-intermediate-bound}
  \K(\nu, \nu_\ast; c) \le
  \frac{g(y) \left[  Z(y) + g(y) \| f c(\cdot, 0) \|_{L^1(\mu)} \right] }{Z(y) Z_\ast(y)} 
  \K(\mu, \mu_\ast; c_y)\,.
\end{equation}
By Assumption~\ref{assumptions-on-phi}(ii) and the definition of $Z, Z_\ast$ we obtain the lower bound
\begin{equation*}
  Z(y) Z_\ast(y) \ge \left( \int_{\mX} h(u,y) \dd \mu(u) \right) \left( \int_{\mX} h(u,y) \dd \mu(u) \right)
  = \| h(\cdot, y) \|_{L^1(\mu)} \| h(\cdot, y) \|_{L^1(\mu_\ast)}.
\end{equation*}
Substituting this bound and \eqref{upper-bound-on-Z-y} back into \eqref{K-prior-pert-intermediate-bound}
completes the proof.
\end{proof}
Similar to the case of likelihood perturbations we
 state an analogue of Theorem~\ref{thm:posterior-prior-perturbations-general} 
with a bound in terms of the model evidences $Z(y)$ and $Z_\ast(y)$ which higlights 
a similar phenomenon that was discussed in Remark~\ref{rem:brittleness}, i.e., 
as $Z(y)$ and $Z_\ast(y)$ vanish our bounds blow up. Indicating the sensitivity of the 
posterior to prior perturbations in mis-specified models.

\begin{corollary}\label{cor:posterior-prior-perturbations-general}
          The result of Theorem~\ref{thm:posterior-prior-perturbations-general} is also 
        true with the bound \eqref{K-prior-pert-intermediate-bound} which is 
        explicit in terms of the constants $Z(y)$ and $Z_\ast(y)$.
\end{corollary}


\section{Applications}\label{sec:applications}

In this section we use Theorems \ref{thm:posterior-likelihood-perturbations-general} and
\ref{thm:posterior-prior-perturbations-general} for different settings to show
their wide applicability and also demonstrate different ways in which these results
can be used to obtain quantitative rates for various situations in which perturbed posteriors
arise due to changes in the prior or the likelihood term. In Subsection~\ref{sec:empirical-prior}, we consider a typical inverse problem for which we are interested to quantify perturbations in the Bayesian posterior when the prior is approximated by an empirical measure using finitely many samples. In Subsection~\ref{sec:matern-prior}, we consider Gaussian process regression using Mat\'ern kernels, an inverse problem for an infinite dimensional parameter space. In Subsection~\ref{sec:pushforward-prior}, we consider perturbations of priors that arise from perturbing the pushforward map to a fixed reference measure, a setting that is common in generative models for instance. Finally, Subsection~\ref{sec:surrogate-models} treats deep neural network surrogate models for Bayesian inversion.

\subsection{Regression with Empirical Priors}\label{sec:empirical-prior}
Let us consider a  regression problem with forward model of the form 
$    y = G(u) + \xi$,
where $ G: \R^d \to \R^d$ is the map 
$(u_1, \dots, u_d) \to ( \tanh(u_1), \dots, \tanh(u_d) )$, $y \in \R^d$ and 
$\xi \sim \N(0, \sigma^2 I_d)$. 
This model leads to the  likelihood potential 
\begin{equation*}
    \Phi(u;y) = \frac{1}{2\sigma^2} | G(u) - y |^2,
\end{equation*}
and we can verify that this potential satisfies Assumption~\ref{assumptions-on-phi} 
with $f = g = 1$ and 
 $h(u,y) = \exp( - (d +|y|^2)/\sigma^2)$ since $|G(u)|\le \sqrt{d}$.
Let us now consider 
a  prior $\mu$ on $u$ as well as $\mu_N$, the empirical 
approximation to $\mu$ obtained from $N$ i.i.d. samples. Our goal is to 
control the error between the resulting true posterior $\nu$ defined by Bayes' 
rule \eqref{bayes-rule} with the prior $\mu$ and the 
posterior $\nu_N$ arising from $\mu_N$. 

Let us choose the cost function $c(u,v) = | u - v|$, in which case, $\D(\cdot, \cdot; \cc)$
coincides with the Wasserstein-1 distance $\W_1(\cdot, \cdot)$. Observe that 
$G\in \Lip(\R^d; \cc)$, and hence 
$\Phi(\cdot;y)$ is globally Lipschitz for any $y\in\mY$ with Lipschitz constant $L(y) = (\sqrt{d}+|y|)/\sigma^2$. 
We can now choose 
\begin{equation*}
    c_y(u,v) =  [1\vee L(y)] [1 \vee |u| \vee |v|] \: |u -v|. 
\end{equation*}
and apply Theorem~\ref{thm:posterior-prior-perturbations-general} 
together with the identity \eqref{K-W-duality} to get 
\begin{equation*}
    \W_1(\nu, \nu_N) \le 
    \frac{1 + \mu( |\cdot |)}{ \exp\left( - \frac 2{\sigma^2}( d + | y|^2)\right) }
    \W(\mu, \mu_N; c_y). 
\end{equation*}
Using Cauchy-Schwartz we can write, for any $\pi \in \Pi(\mu, \mu_N)$,
\begin{equation*}
    \begin{aligned}
    \pi(  c_y(u,v) ) & \le L(y) \left( \pi \left( [1 \vee |u| \vee |v|]^2 \right)  \right)^{1/2}
    \left( \pi( (|u -v|)^2  \right)^{1/2} \\
    & \le  L(y)\left( \pi\left( 1 + |u|^2 + |v|^2 \right)  \right)^{1/2}
    \left( \pi( (|u -v|)^2  \right)^{1/2} \\
    & =  L(y)\left( 1 +  \mu( | \cdot |^2) + \mu_N( | \cdot |^2)\right)^{1/2}
    \left( \pi( (|u -v|)^2  \right)^{1/2}. 
    \end{aligned}
\end{equation*}
Taking the infimum over $\pi$ we 
infer that 
\begin{equation*}
\W(\mu, \mu_N; c_y) \le L(y)\left( 1 +  \mu( | \cdot |^2) + \mu_N( | \cdot |^2)\right)^{1/2} 
\W_2( \mu, \mu_N),
\end{equation*}
with $\W_2$ denoting the Wasserstein-2 
distance. This yields the bound 
\begin{equation}\label{eq:w-1-w-2-bound}
        \W_1(\nu, \nu_N) \le 
    \frac{ L(y)\left[ 1 + \mu( |\cdot |) \right] \cdot \left( 1 +  \mu( | \cdot |^2) + \mu_N( | \cdot |^2)\right)^{1/2}}{ \exp\left( - \frac 2{\sigma^2}( d + | y|^2)\right) }
    \W_2(\mu, \mu_N).
\end{equation}
Let us now take the expectation of the above expression with respect to the 
i.i.d. empirical samples that produce $\mu_N$ to get 
\begin{equation*}
\begin{aligned}
        \EE \W_1(\nu, \nu_N) & \le 
 \EE  \left[   \frac{ L(y)\left[ 1 + \mu( |\cdot |) \right] \cdot 
 \left( 1 +  \mu( | \cdot |^2) + \mu_N( | \cdot |^2)\right)^{1/2}}{ \exp\left( - \frac 2{\sigma^2}( d + | y|^2)\right)}
    \W_2(\mu, \mu_N) \right] \\
    &\le \frac{L(y)(1 + \mu( |\cdot |))}{\exp\left( - \frac 2{\sigma^2}( d + | y|^2)\right)} 
    \left[ 1 +  \mu( | \cdot |^2) + \EE \mu_N( | \cdot |^2) \right]^{1/2} 
    \left[ \EE \W^2_2(\mu, \mu_N) \right]^{1/2}.
\end{aligned}
\end{equation*}
We can now apply \cite[Thm.~1]{fournier2015rate} which gives a sharp 
rate for $\EE \W^2_2(\mu, \mu_N)$. In particular, assuming $\mu( | \cdot |^3)< + \infty$, we obtain the bound 
\begin{multline*}
    \left[ \EE \W_1(\nu, \nu^N) \right]^2 \\
    \le C(\mu) \frac{(\sqrt{d}+|y|)^2\exp\left[\frac 4{\sigma^2}\left(d +|y|^2\right)\right]}{\sigma^4}
    \times \left\{ 
    \begin{aligned}
      & N^{-1/2} + N^{-1/3} && \text{if } d < 4,\\ 
      & N^{-1/2} \log(1 + N) + N^{-1/3} && \text{if } d = 4, \\ 
      & N^{-2/d} + N^{-1/3}&& \text{if } d > 4,
    \end{aligned}
    \right.
\end{multline*}
where $C(\mu) <+\infty$ is a constant depending on the first three moments of $\mu$.

\subsection{Gaussian Process Regression with Mat\'ern Kernels}\label{sec:matern-prior}
Here we consider a non-parametric regression problem with a Gaussian process 
prior belonging to the Mat\'ern class. The perturbation of 
the posterior is controlled by the hyper-parameters defining the Gaussian process prior.  Let
$\Omega = [0,1]^d$ and let $\kappa_i \in C^\infty_c(\Omega)$ for $i=1, \dots, m$ be fixed functions (filters)  and define the forward map
\begin{equation*}
	G: u \mapsto \left( \langle \kappa_1, u \rangle, \dots, \langle \kappa_m, u \rangle \right)^\top,
\end{equation*}
for  functions $u \in L^2(\Omega)$ and with $\langle \cdot, \cdot \rangle$ denoting the 
$L^2$--inner product.
 Once again we consider the model $y = G(u) + \xi$ 
 where $\xi \sim \N(0, \sigma^2 I_m)$,
leading to the likelihood 
\begin{equation*}
  \Phi(u;y) = \frac{1}{2\sigma^2} | G(u) - y |^2.
\end{equation*}
We take $\mX = L^2(\Omega)$ and observe that 
$G: \mX \to \R^m$ is bounded and linear.
 As for the prior we choose 
\begin{equation*}
  \mu = \N( 0,  \gamma^2(\Delta + \tau\Id)^{-2\alpha} ),
\end{equation*}
with $\Delta$ and $\Id$ denoting the Laplacian and identity operator on $\mX$, respectively, and with constants $\gamma, \tau >0$
and integer $\alpha >0$. The above Gaussian prior defines a Gaussian field 
on $\Omega$ with a Mat\'ern covariance function \cite{whittle1963stochastic}.
In practice the hyper-parameters $\gamma, \tau$ are tuned using 
various techniques such as empirical Bayes, maximum likelihood, 
or cross validation \cite{williams2006gaussian}. The parameter $\alpha$ can also be considered 
as a hyper-parameter but here we consider it  fixed for convenience.
We do not consider hyper-parameter tuning strategies here but instead wish to 
control the distance between the resulting posteriors for different 
choice of these hyper-parameters. 
To this end, 
let $\gamma_\star, \tau_\star$ be a second set of hyper-parameters 
with corresponding Mat\'ern prior $\mu_\star$ and let $\nu$ and $\nu_\star$ 
denote the  posteriors that arise from $\mu$ and $\mu_\star$
with the same likelihood potential outlined above. Our goal below is to bound $\W_1(\nu, \nu_\star)$.

Since $G$ is bounded and linear we can verify that $\Phi$ 
satisfies Assumption~\ref{assumptions-on-phi} with  $f = g = 1$ and 
$h(u,y) = \exp( - \frac{1}{\sigma^2} \left[ \| G\| \cdot  \|u\|_{\mX}^2 + |y|^2 \right])$ where $\| G\|$ 
denotes the operator norm of $G$. 
The function $h$ is then integrable under any Gaussian prior thanks to Fernique's theorem. 
Let us now take $c(u,v) = \|u - v\|_\mX$. 
By direct calculation we can verify that $\Phi$ satisfies the 
assumption of Theorem~\ref{thm:posterior-prior-perturbations-general} 
with 
$L(u,v; y) = \frac{\| G\|}{\sigma^2}\left( \|u\|_\mX + \|v\|_\mX + |y| \right) $.
Thus Fernique's theorem along with  Theorem~\ref{thm:posterior-prior-perturbations-general} imply 
$\W_1(\nu, \nu_\star) \le C \D(\mu, \mu_\star; c_y)$ where we can take
\begin{equation*}
    c_y(u,v) =  [ 1 \vee \|u\|_\mX \vee \|v\|_\mX] 
    \cdot [ 1 \vee L(u,v;y)] \| u - v \|_\mX\,.  
\end{equation*}
We can absorb the coefficient $\frac{\| G \|}{\sigma^2}$ from $ L(u,v;y)$
into the constant $C$ to get $\W_1( \nu, \nu_\star) \le C \K (\mu, \mu_\star; c'_y)$ where $c'_y(u,v) = (1 +  \| u\|_\mX + \|v\|_\mX + | y| )^2 \| u-v \|_\mX$.  For fixed $y\in \mY$, this cost is equivalent 
to $c_y$ but it is more convenient to work with. Thus, it remains for us 
to bound $\K( \mu, \mu_\ast; c'_y)$ and we do this using a coupling argument. 

Since $\mu, \mu_\star$ are Gaussian measures they have Karhunen-Lo\'eve 
expansions 
\begin{equation*}
    \begin{aligned}
      \mu = \Law\left\{ \sum_{j=1}^\infty \sqrt{\lambda_j} \xi_j x_k  \right\}, 
      \qquad \xi_j \iidsim \N(0,1), \\ 
      \mu_\star = \Law\left\{ \sum_{j=1}^\infty \sqrt{\lambda^\star_j} \xi_j x_k  \right\}, 
      \qquad \xi_j \iidsim \N(0,1),
    \end{aligned}
\end{equation*}
where $x_j$ are the eigenfunctions of $\Delta$ on $\Omega$ and
\begin{equation*}
    \lambda_j = \gamma^2 \left( \tilde{\lambda}_j + \tau \right)^{-2\alpha}, \qquad 
        \lambda^\star_j = \gamma_\star^2 \left( \tilde{\lambda}_j + \tau_\star \right)^{-2\alpha},
\end{equation*}
with $\tilde{\lambda}_j$ denoting the eigenvalues of $\Delta$. 
Using Lemma~\ref{lem:product-prior-coupling} below we obtain the bound 
\begin{equation}\label{eq:w-2-hierarchical-dsp}
    \K(\mu, \mu_\star; c'_y) \le C  \sum_{j=1}^\infty \left| \sqrt{\lambda_j} 
    - \sqrt{\lambda_j^\star} \right|^2,
\end{equation}
for some constant $C> 0$ that depends on the moments of $\mu$ and $\mu_\star$. 
To control the error in terms of the hyper-parameters is now a matter of algebra. 
We can write 
\begin{equation*}
    \left| \sqrt{\lambda_j} - \sqrt{\lambda_j^\star}\right|
    \le  ( \tilde \lambda_j + \tau)^{-\alpha} | \gamma - \gamma_\star| 
    + \gamma_\star \left| \left( \tilde{\lambda}_j + \tau \right)^{-\alpha} 
    - \left( \tilde{\lambda}_j + \tau_\star \right)^{-\alpha}
    \right|.
\end{equation*}
We can further bound the second term using the mean value theorem to get 
\begin{equation*}
    \left| \sqrt{\lambda_j} - \sqrt{\lambda_j^\star}\right|
    \le  ( \tilde \lambda_j + \tau)^{-\alpha} | \gamma - \gamma_\star| 
    + \alpha \gamma_\star  \left( (\tilde\lambda_j  + \tau)^{- \alpha - 1} \vee 
    (\tilde\lambda_j + \tau_\star)^{- \alpha - 1}  \right)  | \tau - \tau_\star|.
\end{equation*}
Substituting back into \eqref{eq:w-2-hierarchical-dsp} yields the bound 
\begin{equation*}
   \left[  \W_1(\nu, \nu^\star) \right]^2 \le C \left( |\gamma - \gamma_\star| 
   + |\tau - \tau_\star|   \right)^2, 
\end{equation*}
where the constant $C> 0$ depends on the choice of $\alpha, \gamma, \gamma_\star, \tau, \tau_\star$ as well as the $\tilde{\lambda_j}$; it is possible to derive this constant 
explicitly but we do not pursue this route for brevity.


  


Let us now present the technical lemma that led to the bound \eqref{eq:w-2-hierarchical-dsp} above
which is also independently interesting.
The result provides a general and convenient approach for bounding  
$\W(\cdot, \cdot; c)$ distances between product measures 
for a broad family of cost functions $c$ using a simple coupling argument.

\begin{lemma}\label{lem:product-prior-coupling}
  Let $\mX$ be a separable Hilbert space and $\mu, \mu_\ast \in \PP(\mX)$ be product measures of the form 
  \begin{equation*}
    \begin{aligned}
      \mu & = \Law \left\{ \sum_{j=1}^\infty \sqrt{\lambda_j} \xi_j x_j \right\},  \quad
      \xi_j \iidsim \eta, \\
 \mu_\ast& = \Law \left\{ \sum_{j=1}^\infty \sqrt{\lambda^\ast_j} \xi^\ast_j
 x^\ast_j \right\}, \quad \xi_j^\ast \iidsim \eta_\ast, 
\end{aligned}
  \end{equation*}
where $\{ \lambda_j\}_{j=1}^\infty$ and $\{ \lambda^\ast_j\}_{j=1}^\infty$ are fixed
real valued sequences, $\{x_j\}_{j=1}^\infty$ and $\{x_j^\ast\}_{j=1}^\infty$
are  orthonormal bases in $\mX$ and  $\eta, \eta_\ast \in \PP(\mbb R)$ have mean zero and
unit variance. Suppose both random
sums converge a.s. in $\mX$  so that $\mu$ and $\mu_\ast$ are well-defined and define
$c: \mX \times \mX \to \mbb R_{\ge 0 }$ to  be a lower-semicontinuous semi-metric of the form 
\begin{equation*}
  c(u,v) = \Big( \| u\|_\mX + \| v\|_\mX \Big)^s \| u -v \|_\mX,
\end{equation*}
with $s \ge 0$. Then it holds that 
\begin{equation*}
  \begin{aligned}
  \K (\mu, \mu_\ast; c) & \le
  2^{\frac{1 \vee (2 s -1)}{2} + 1}
  \left( \E_\mu\| u\|_\mX^{2s} + \E_{\mu_\ast} \| u_\ast \|_\mX^{2s} \right)^{\frac{1}{2}} \\
  & \times
  \left(  \W_2^2(\eta, \eta_\ast) \sum_{j=1}^\infty \lambda_j  +    \sum_{j=1}^\infty  
  \lambda_j \| x_j - x_j^\ast\|_\mX^2 + \sum_{j=1}^\infty \left( \sqrt{\lambda_j} - \sqrt{\lambda_j^\ast} \right)^2 \right)^{\frac{1}{2}}.
\end{aligned}
\end{equation*}
\end{lemma}

\begin{proof}
  Let $\varpi_0$ be the optimal coupling between $\eta, \eta_\ast$ that attains
  $\W_2(\eta, \eta_\ast)$. If such a coupling does not exist then we set $\W_2(\eta, \eta_\ast) = +\infty$
  and the bound is trivial and so henceforth we assume $\W_2(\eta,  \eta_\ast) < +\infty$. 
  Now let $\pi_0$ be a coupling between $\mu$ and $\mu_\ast$ given as the law
  of the random variable $(u, u_\ast) \in \mX \times \mX$  constructed as follows:
  \begin{equation*}
    u = \sum_{j=1}^\infty \sqrt{\lambda_j} \xi_j x_j, \qquad u_\ast =
    \sum_{j=1}^\infty \sqrt{\lambda^\ast_j} \xi^\ast_j
 x^\ast_j, \qquad (\xi_j, \xi_j^\ast) \iidsim \varpi_0.
\end{equation*}
By  duality \eqref{D-le-W} we have that
\begin{equation*}
  \begin{aligned}
  \K(\mu, \mu_\ast; c)
  & \le \int_{\mX \times \mX} c(u,u_\ast) \dd \pi_0(u, u_\ast)\\
  & = \E_{\pi_0} c \left( \sum_{j=1}^\infty \sqrt{\lambda_j} \xi_j x_j,
    \sum_{j=1}^\infty \sqrt{\lambda^\ast_j} \xi^\ast_j x^\ast_j \right) \\
  & \le 2^{\frac{1 \vee (2s-1)}{2}} \left( \E_\mu\| u\|_\mX^{2s} + \E_{\mu_\ast} \| u_\ast \|_\mX^{2s} \right)^{1/2}
  \left( \E_{\pi_0} \| u - u_\ast \|_\mX^2 \right)^{1/2}.
\end{aligned}
\end{equation*}
We further have, by Parseval's identity
\begin{equation*}
  \begin{aligned}
    \E_{\pi_0} & \| u - u_\ast \|_\mX^2
     =  \E_{\pi_0}  \left\| \sum_{j=1}^\infty  \sqrt{\lambda_j} \xi_j x_j - \sqrt{\lambda_j^\ast} \xi_j^\ast
        x_j^\ast  \right\|_\mX^2   \\
      & \le 4  \E_{\pi_0} \left[ \sum_{j=1}^\infty ( \xi_j- \xi_j^\ast)^2 \lambda_j
        +   \left\| \sum_{j=1}^\infty \xi^\ast_j \sqrt{\lambda_j}  \left(  x_j - x_j^\ast \right) \right\|_\mX^2
        + \sum_{j=1}^\infty (\xi^\ast_j)^2 \left( \sqrt{\lambda_j}  - \sqrt{\lambda_j^\ast} \right)^2   \right]  \\
      & = 4  \left( \sum_{j=1}^\infty \E_{\varpi_0} | \xi_j - \xi_j^\ast|^2 \lambda_j
      +  \E_{\eta_\ast}   \left\| \sum_{j=1}^\infty \xi^\ast_j \sqrt{\lambda_j}  \left(  x_j - x_j^\ast \right) \right\|_\mX^2
    +   \sum_{j=1}^ \infty (\sqrt{\lambda_j} - \sqrt{\lambda_j^\ast})^2 \right) \\
    & \le  4 \left(  \W_2^2(\eta, \eta_\ast) \sum_{j=1}^\infty \lambda_j + 
     \E_{\eta_\ast}   \left(  \sum_{j=1}^\infty \xi^\ast_j \sqrt{\lambda_j}  \left\|  x_j - x_j^\ast \right\|_\mX \right)^2
    +  \sum_{j=1}^\infty (\sqrt{\lambda_j} - \sqrt{\lambda_j^\ast})^2 \right) \\
    & =   4 \left(  \W_2^2(\eta, \eta_\ast) \sum_{j=1}^\infty \lambda_j +   \sum_{j=1}^\infty  \lambda_j  \left\|  x_j - x_j^\ast \right\|_\mX^2
    +  \sum_{j=1}^\infty (\sqrt{\lambda_j} - \sqrt{\lambda_j^\ast})^2 \right).
  \end{aligned}
\end{equation*}
Substituting  back into the upper bound on $\K$ yields the desired result.
{}
\end{proof}

\subsection{Pushforward Priors}\label{sec:pushforward-prior}

Below we consider an abstract example where the prior measure $\mu$ is identified as 
the pushforward of a fixed reference measure, while $\mu_\ast$ is given by 
the pushforward of the same reference measure through a perturbed map. 
This problem is at the heart of data driven techniques for learning of priors from 
data that have  recently become popular \cite{patel2022solution, gonzalez2022solving, laumont2022bayesian} due to the incredible success of 
generative machine learning models such as generative adversarial networks (GANs) \cite{goodfellow2020generative}
and normalizing flows \cite{kobyzev2020normalizing}.

Let $\mX = [0,1]^d$
and consider once again an inverse problem with a linear 
forward mapping $G : \mX \to \R^m$ giving rise to the likelihood potential
\begin{equation*}
	\Phi(u; y) = \frac{1}{2\sigma^2} |G(u) - y|^2.
\end{equation*}
 Observe that, as in Section~\ref{sec:matern-prior}, $\Phi$ satisfies assumption \ref{assumptions-on-phi} with $f\equiv g \equiv 1$ and 
 $h(u,y) = \exp\left( - \frac{1}{\sigma^2} \left[  \|G\|_\infty^2 \cdot |u|^2 + |y|^2 \right] \right)  $.
 In fact, since $\mX$ is compact we infer that  we can 
take  the Lipschitz constant of $\Phi$ to be $L(u,v; y) = C( 1 + |y|)$ for a constant $C> 0$. 
Then repeating the same calculation that led to \eqref{eq:w-1-w-2-bound} we 
realize that bound holds in this example without any modifications.

Now suppose $\mu = T_\sharp \varrho$, i.e., the pushforward of $\varrho$ 
through the map $T: \mX \to \mX$, 
where $\varrho \in \PP(\mX)$ is a reference measure. 
Similarly let $\mu_\ast = T_\sharp^\ast \varrho$ where $T^\ast:\mX \to \mX$ is a second
transport map. In the aforementioned applications, we think of  $T^\ast$
as an approximation to $T$. Therefore, it is natural to try to obtain 
a bound on $\W_1(\nu, \nu_\ast)$ in terms of the distance between $T$ and $T^\ast$.
The key to obtaining such a result is stability estimates for Wasserstein distances in terms 
 of the $L^p$ distances between the maps. We recall the following result:

\begin{proposition}[{\cite[Thm. 2]{sagiv2019wasserstein}}]
\label{thm:sagiv}
Let $\mX = [0,1]^d$  and suppose that $T, T^\ast \in L^p(\varrho) \bigcap C(\mX)$ 
for some $p\geq 1$. Then  
$\W_p(\mu, \mu_\ast) \leq \|T - T^\ast\|_{L^p(\varrho)}.$
\end{proposition}

Applying the above result with $p=2$ to the bound \eqref{eq:w-1-w-2-bound} we immediately obtain 
\begin{equation*}
    \W_1(\nu, \nu_\ast) 
    \le \frac{C(1+|y|)(1 + \mu( |\cdot |))}{\|h(\cdot,y)\|_{L^1(\mu)} \|h(\cdot,y)\|_{L^1(\mu_\ast)}} 
    \left( 1 +  \mu( | \cdot |^2) + \mu_\ast( | \cdot |^2) \right)^{1/2}
    \|T - T^\ast\|_{L^2(\varrho)}.
  \end{equation*}
The quantity $\|T - T^\ast\|_{L^2(\varrho)}$ can be controlled using off-the-shelf 
results from approximation theory. For example, one can take $T^\ast$ to be a polynomial \cite{devore1993constructive}, 
 wavelet \cite{meyer1992wavelets}, neural network \cite{devore2021neural}, or kernel approximation \cite{wendland2004scattered}
 to $T$ and, assuming  sufficient regularity, obtain quantitative rates.

  


\subsection{DNN Surrogate Models}\label{sec:surrogate-models}
Bayesian inversion of partial differential equations has gained significant attention in statistics in recent years, particularly in the area of numerical treatment of uncertainty quantification with random field inputs. This is driven by the need for large-scale computational solutions in science and engineering.
A random series expansion of the input data gives rise to an infinite-dimensional parametric Bayesian inverse problem for which efficient quadrature methods have been studied \cite{schwab-sparse}. For instance, under the assumption of an additive Gaussian noise model for the observable data, it has been demonstrated that the likelihood potential in inverse problems arising from elliptic and parabolic PDEs has a holomorphic extension 
into a subset of the complex domain \cite{schillings2013sparse, schillings2014sparsity}. The quantification of this subset gives rise to explicit sparse spectral approximation strategies of the likelihood and posterior distributions.

Let us here consider feed-forward DNN surrogates for likelihood potentials 
$\Phi(\cdot; y) : [0,1]^d \to \R$ that have a holomorphic extension to
a subset of
${\mathbb C}^d$. A feed-forward DNN has the architecture of an iterated composition of affine transformations followed by a nonlinear activation function. 
To be more precise, a DNN $\Phi^N : \R^d \to \R$ approximating the likelihood potential $\Phi$ can be expressed as
\begin{equation}
	\label{eq:def_DNN}
	\Phi^N(u; y) = W_L \sigma(\ldots\sigma(W_2\sigma(W_1 u + b_1) + b_2)\ldots)+b_L, \quad u\in \R^d,
\end{equation}
where $W_j \in \R^{d_{j+1}\times d_{j}}$ and $b_j \in \R^{d_{j+1}}$ for $j=1,...,L$, where $d_1 = d$ and $d_{L+1} = 1$. Note that here we fixed $y$ and in general the parameters of the network may 
depend on the data $y$ \footnote{One can also parameterize the network as a 
function $\Phi: \mX \times \mY \to \R$ but we will not to this here for simplicity}. The function $\sigma : \R^d \to \R^d$ is defined for any $d$ based on a monotone activation function $\sigma: \R \to \R$, which is applied entrywise to any vector. 
Note that the input dimension of $\sigma$ in equation \eqref{eq:def_DNN} varies from one 
layer to the next.
The number of hidden layers $L-1$ is referred to as the {\it depth} and $N$, the total number of nonzero components in $W_j$ and $b_j$, $j=1,...,L$, as the {\it size} of the DNN.
We will further restrict our attention to Rectified Linear Unit (ReLU) activation functions $\sigma(u) = \max\{0, u\}$ as this guarantees that the results from \cite{herrmann2020deep} can be applied (see Proposition~\ref{thm:hermann} below).

For simplicity, we assume below that $\Phi\geq 0$ so that
Assumption \ref{assumptions-on-phi} is satisfied with $f\equiv g\equiv 1$. Moreover, let us consider 
a fixed observational data set $y \in \R^m$. 
In accordance to \cite{opschoor2022exponential}, we assume that $\Phi$ has aa holomorphic extension to a Bernstein polyellipse
\begin{equation}
    \label{eq:polyellipse}
    {\mathcal E}_{\boldsymbol{\rho}} := {\mathcal E}_{\rho_1} \times ... \times {\mathcal E}_{\rho_d} \subset {\mathbb C}^d,
\end{equation}
where $\boldsymbol{\rho} = (\rho_1, ..., \rho_d)$ and ${\mathcal E}_{\rho_j} \subset \mathbb{C}$ stands for an ellipse containing $[-1,1]$ with semiaxis sum $\rho_j>1$.
First, we recall an existence  result for DNNs, stating that there exist DNNs of a 
requisite depth and size that can approximate holomorphic likelihoods exponentially fast 
in the size of the network:
\begin{proposition}[{\cite[Thm. 3.6]{opschoor2022exponential}}]
\label{thm:hermann}
Fix $y\in\mY$ and
suppose $\Phi(\cdot;y) : [-1,1]^d \to \R$ has an holomorphic extension to ${\mathcal E}_{\boldsymbol{\rho}}$ given in \eqref{eq:polyellipse}. 
Then for any $M\in \mathbb{N}$ there exists a ReLU DNN $\Phi^N : [0, 1]^d \to \R$ such that
\begin{equation*}
	\|\Phi(\cdot; y) - \Phi^N(\cdot;y)\|_{W^{1,\infty}([-1,1]^d)}\leq C \exp(-\kappa M^{\frac 1{d+1}})\,,
\end{equation*}
where $\kappa = \kappa(\boldsymbol{\rho}, d) > 0$
and $C = C(\Phi, \boldsymbol{\rho}, d)>0$ are constants and 
\begin{equation*}
	L \leq CM^{\frac 1{d+1}} \log_2 M \quad \text{and} \quad 
	N  \leq M.
\end{equation*}
\end{proposition}
We now wish to show that the above exponential convergence rate can be extended 
to a bound on $\K(\nu, \nu^N; c)$ where $\nu^N$ is the posterior 
that arises by replacing $\Phi$ with $\Phi^N$.
To this end, let us now consider a prior $\mu \in \PP([0,1]^d)$ and distance-like function $c$ 
such that $c(\cdot, 0) \in L^1(\mu)$ and let $h(\cdot, y)$ be the function in Assumption~\ref{assumptions-on-phi} and let $\Phi^N$ be the DNN that was specified in 
Proposition~\ref{thm:hermann} above. 
We can then compute
\begin{eqnarray*}
	\Phi^N(u; y) & \leq & |\Phi^N(u;y) - \Phi(u;y)| + \Phi(u;y) \\
  & \leq & C \exp(-\kappa M^{\frac 1{d+1}}) - \log h(u;y) \\
	& \leq  & - \log \left(\widetilde{C} h(u;y)\right).
\end{eqnarray*}
At the same time $\Phi^N(u;y) \geq - C\exp(-\kappa M^{\frac 1{d+1}})$ and so 
$\Phi^N$ satisfies Assumption~\ref{assumptions-on-phi} with a function $h^N$
that is proportional to $h$ and with a function $f^N$ that is constant\footnote{Note that the 
function $g$ is fixed here since the data $y$ is assumed fixed}. 
Then an application of Theorem~\ref{thm:posterior-likelihood-perturbations-general}
yields the bound 
\begin{equation*}
	\K(\nu, \nu^N; \cc) \le
    C \| c(\cdot, 0) \|_{L^1(\mu)}\exp\left(-\kappa M^{\frac 1{d+1}}\right),
\end{equation*}
with constant $C=C(\Phi,\kappa,y) > 0$.



\section{Conclusion}
We presented novel theoretical results that control the perturbations of 
Bayesian posterior measures $\nu$ with respect to those of likelihood potentials $\Phi$ and 
the prior measures $\mu$ with respect to IPMs, a broad category of divergences
 that encompasses Wasserstein metrics among others.
The proposed results are 
theoretically convenient as they allow us to adapt the IPMs to the problem at hand 
and as a result can accommodate large classes of nonlinear and locally Lipschitz 
likelihood potentials. 

The use of IPMs is particularly well adapted when prior perturbations are concerned since these 
divergences do not require the perturbed priors to remain absolutely continuous with 
respect to the original prior. This flexibility opens the door to analyzing interesting 
approximation methods such as empirical priors built from a set of samples,
Gaussian priors with very different covariance kernels, and pushforward prior, as 
demonstrated by our applications. These results further pave the path towards 
convergence analysis of modern machine learning techniques for BIPs such as 
data driven priors or DNN approximations.



In future work, it would be interesting to extend the prior and likelihood perturbation results such as presented here to the frequentist framework and study average stability with respect to a 'true' data-generating distribution $y|u^\dagger$. Moreover, the applications discussed in Section~\ref{sec:applications} give rise to interesting further questions such as regression in non-parametric setting or distances between hierarchical Gaussian fields.

\section{Acknowledgements}
AGI is supported by {\it Asociación Mexicana de Cultura, A.C.} and is partially supported by Mexico's {\it Sistema Nacional de Investigadores}. The work of TH was supported by the the Academy of Finland (decision 326961) and the Research Foundation of LUT University. 
FH is supported by start-up funds at the California Institute of Technology and the Deutsche Forschungsgemeinschaft (DFG, German Research Foundation) via project 390685813 - GZ 2047/1 - HCM.
BH is supported by the National Science Foundation grant DMS-208535.  

\bibliographystyle{abbrv}
\bibliography{PosteriorPert_ref}

\appendix

\end{document}